\documentclass{article}
\usepackage{kr}

\usepackage{times}
\usepackage{soul}
\usepackage{url}
\usepackage[hidelinks]{hyperref}
\usepackage[utf8]{inputenc}
\usepackage[small]{caption}
\usepackage{graphicx}

\usepackage{amsmath}
\usepackage{mathtools} 
\usepackage{amsthm}
\usepackage{booktabs}
\usepackage{algorithm}
\usepackage{algorithmic}
\usepackage{bm}
\usepackage{bbold}
\usepackage{stmaryrd}
\usepackage{tikz}
\usetikzlibrary{positioning}
\usetikzlibrary{matrix,calc}

\usepackage{xcolor}
\usepackage{eurosym}
\usepackage{graphicx}
\usepackage{amssymb}
\usepackage{amsfonts}

\def\cnf{{\tt CNF}}
\def\dnf{{\tt DNF}}
\def\dt{{\tt DT}}
\def\Dnnf{{\tt DNNF}}

\def\c2d{{\tt C2D}}
\def\mc2d{{\tt mini-C2D}}
\def\d4{{\tt D4}}

\def\L{$\mathcal{L}$}

\def\co{$\mathbf{CO}$}

\def\se{$\mathbf{SE}$}
\def\im{$\mathbf{IM}$}
\def\eq{$\mathbf{EQ}$}
\def\ct{$\mathbf{CT}$}

\def\me{$\mathbf{ME}$}
\def\mc{$\mathbf{MC}$}
\def\opt{$\mathbf{OPT}$}

\def\cd{$\mathbf{CD}$}
\def\fo{$\mathbf{FO}$}

\def\nc{\mathbf{\neg C}}

\newtheorem{prop}{Proposition}
\newtheorem{proposition}[prop]{Proposition}
\newtheorem{defn}{Definition}
\newtheorem{definition}[defn]{Definition}

\makeatletter
\def\moverlay{\mathpalette\mov@rlay}
\def\mov@rlay#1#2{\leavevmode\vtop{%
   \baselineskip\z@skip \lineskiplimit-\maxdimen
   \ialign{\hfil$\m@th#1##$\hfil\cr#2\crcr}}}
\newcommand{\charfusion}[3][\mathord]{
    #1{\ifx#1\mathop\vphantom{#2}\fi
        \mathpalette\mov@rlay{#2\cr#3}
      }
    \ifx#1\mathop\expandafter\displaylimits\fi}
\makeatother

\newcommand{\bigcupdot}{\charfusion[\mathop]{\bigcup}{\cdot}}

\newcommand{\argmin}{\mathrm{argmin}}
\newcommand{\concept}[1]{\llbracket #1 \rrbracket}
\newcommand{\size}[1]{|#1|}

\renewcommand{\vec}[1]{\bm{#1}}

\newtheorem{example}{Example}

\title{On the Computational Intelligibility of Boolean Classifiers}

\author{
Gilles Audemard$^1$\and
Steve Bellart$^1$\and
Louenas Bounia$^1$\and
Frédéric Koriche$^1$\and \\
Jean-Marie Lagniez$^1$\and
Pierre Marquis$^{1, 2}$\\ 
\affiliations
$^1$CRIL, Université d'Artois \& CNRS, France\\
$^2$Institut Universitaire de France\\
\emails
\{audemard, bellart, bounia, koriche, lagniez, marquis\}@cril.fr
}

\begin{document}

\maketitle

\begin{abstract}
	In this paper, we investigate the \emph{computational intelligibility} of Boolean classifiers,
	characterized by their ability to answer XAI queries in polynomial time.
	The classifiers under consideration are decision trees, \dnf\ formulae, decision lists, decision rules, tree ensembles, and
	Boolean neural nets. Using $9$ XAI queries, including both explanation queries and verification queries,
	we show the existence of \emph{large intelligibility gap between the families of classifiers}. On the one hand, all the $9$ XAI queries
	are tractable for decision trees. On the other hand, none of them is tractable for \dnf\ formulae, decision lists, random forests, boosted decision trees,
	Boolean multilayer perceptrons, and binarized neural networks.
\end{abstract}

\section{Introduction}\label{sec:intro}

\emph{What is a good classifier?} Such a common question calls for the identification
of several criteria, in order to assess the quality of classifiers.
To this point, a key criterion for measuring the generalization ability of classifiers is \emph{accuracy}.
Given a probability distribution over data instances, the accuracy of a (Boolean) classifier
is defined by the probability of correctly labeling a random data instance.
In statistical learning, the underlying probability distribution is unknown,
and we only have access to a data sample for training the classifier.
The learning problem is thus cast as a stochastic optimization task: given a family of candidate
classifiers, often referred to as the concept class, find a classifier in the family that minimizes
the (possibly regularized) empirical error on the training sample. In practice, the classification accuracy
is estimated on test samples using evaluation metrics such as, for example, stratified cross-validation.

However, accuracy is not the sole criterion for choosing a classifier: in many real-world applications,
another important criterion is \emph{intelligibility}.
Roughly speaking, a classifier is intelligible if its predictions
can be \emph{explained} in understandable terms to a user, and if its behavior can
be \emph{verified} according to the user's expectations. The explainability requirement
is a legal issue in Europe since the implementation of the General Data Protection Regulation (EU) 2016/679 (“GDPR”)
on May 25th, 2018 \cite{DBLP:journals/aim/GoodmanF17}.  Accordingly, explainable and robust AI (XAI) has become a very active research
topic for the past couple of years (see e.g. \cite{Bunel2018,ShihCD18,Plumb2018,Ignatiev2019,ChenZS0BH19,Srinivasan2019,DBLP:conf/sat/ShihDC19,Crabbe2020,Horel2020,Jia2020,DBLP:conf/nips/0001GCIN20,Ramamurthy2020}).

\begin{table*}[ht]
	\begin{center}
		\begin{normalsize}
			\scalebox{0.8}{
			\begin{tabular}{|c|c|}
				\hline
				{\bf XAI query} & {\bf Description}                                                     \\
				\hline
				\hline
				EMC             & Enumerating Minimum-Cardinality explanations                          \\
				DPI             & Deriving one Prime Implicant explanation                              \\
				ECO             & Enumerating COunterfactual explanations                               \\
				\hline
				\hline
				CIN             & Counting the INstances associated with a given class                  \\
				EIN             & Enumerating the INstances associated with a given class               \\
				IMA             & Identifying MAndatory features or forbidden features in a given class \\
				IIR             & Identifying IRrelevant features in a given class                      \\
				IMO             & Identifying MOnotone (or anti-monotone) features in a given class     \\
				MCP             & Measuring Closeness of a class to a Prototype                         \\
				\hline
			\end{tabular}
			}
			\caption{Some XAI queries.}\label{table:summary}
		\end{normalsize}
	\end{center}
\end{table*}

Despite the importance of both criteria, intelligibility appears to be much more difficult to circumscribe than accuracy.
Indeed, in the statistical learning literature, the generalization ability of classifiers has been formally characterized
through the prisms of learnability \cite{Valiant1984,Haussler1992}, uniform convergence \cite{Vapnik1998}, and algorithmic stability \cite{Shalev2010,Charles18}.
By contrast, the term ``intelligibility'' holds no agreed upon meaning,
since it depends on various desiderata for clarifying the classifier behavior in some practical situations \cite{Lipton2018}.
Yet, different \emph{forms} of intelligibility can be formalized, by focusing on the classifier ability to properly answer questions.
Such forms of intelligibility are reflected by explanation and verification queries introduced so far in the XAI literature.
Notably, a classifier should be equipped with \emph{explanation facilities} including, for example,
the ability to identify few relevant features which together are sufficient for predicting the label of a data instance.
Furthermore, the classifier should be \emph{amenable to inspection}, especially when the user has some expectations about
the behavior of the classifier, and she is interested in checking whether the classifier complies to those expectations.
For instance, in a loan classification problem, if a loan is granted to an applicant who does not have a high income,
then loan should not be denied when the income increases, provided that the other features are unchanged.
This expectation can be formalized using a verification query that checks the monotonic behavior
of the classifier on the feature related to the applicant's income.

Addressing such XAI queries requires the availability of inference algorithms for computing answers in reasonable time.
From this perspective, each query can be viewed as a property that a family of classifiers may offer or not:
it is offered when there exists a polynomial-time algorithm to answer the query from any classifier of the family, and it is not
when there is no such algorithm, unless {\sf P = NP}. In other words, the {\it computational intelligibility} of a family of classifiers
can be defined as the set of tractable XAI queries supported by the family, leading to an {\it intelligibility map} when several families of classifiers
are considered. Such an approach echoes the computational evaluation of KR languages
achieved in the knowledge compilation map \cite{DarwicheMarquis02}.

The aim of this paper is to pave the way for the computational intelligibility of Boolean classifiers.
As a baseline, we use $9$ XAI queries from those considered in \cite{Audemardetal20}, which are summarized in
Table~\ref{table:summary}: EMC, DPI, and ECO are explanation queries, and
CIN, EIN, IMA, IIR, IMO, MCP are verification queries.\footnote{Five additional
	verification queries, namely CAM, EAM, MFR, MCJ, MCH, are considered in \cite{Audemardetal20}, but
	they mainly trivialize or boils down to another query in the list when dealing with Boolean classifiers
	-- \cite{Audemardetal20} considers the more general case of multi-label classifiers, i.e., when more than two output classes are targeted.}
%
Based on this portfolio of XAI queries, we examine $7$ families of Boolean classifiers: decision trees, \dnf\ formulae, decision lists,
random forests, boosted decision trees, Boolean multilayer perceptrons, and Boolean neural networks.
The main contribution of this paper lies in a number of complexity results establishing the existence
of a \emph{large intelligibility gap between families of classifiers}, estimated by the number of XAI queries (over $9$) which are tractable.
Specifically, we prove that all XAI queries are tractable for the family of decision trees, while none of them
is tractable for  \dnf\ formulae, decision lists, random forests, boosted trees, Boolean multilayer perceptrons, or of binarized neural networks.


The rest of the paper is organized as follows. In Section \ref{sec:formal} is reported the necessary background about Boolean functions
and their representations. In Section \ref{sec:classifiers} the $7$ families of Boolean classifiers examined in this study are presented.
The $9$ XAI queries summarized in Table~\ref{table:summary} are presented in formal terms in Section \ref{sec:XAIqueries}.
Results are provided in Section \ref{sec:complexity}: for each of the $9$ XAI
queries and each  of the $7$ families of classifiers, we determine whether the family offers or not the query.
Finally, Section \ref{sec:concl} concludes the paper and presents some perspectives for further research.
\section{Formal Preliminaries}\label{sec:formal}

For a positive integer $n$, let $[n]$ to denote the set $\{1,\cdots,n\}$.
Let $\mathcal F_n$ be the set of all Boolean functions from $\mathbb{B}^n$ into $\mathbb{B}$, where $\mathbb{B} = \{0,1\}$.
Any member $f$ of $\mathcal F_n$ is called a \emph{concept}, and any subset $\mathcal F$ of $\mathcal F_n$ is called a \emph{concept class}.
Any vector $\vec x$ in the Boolean hypercube $\mathbb{B}^n$ is called an \emph{instance}; $\vec x$ is a positive example
(or \emph{model}) of some concept $f$
if $f(\vec x) = 1$, and $\vec x$ is a negative example of $f$ if $f(\vec x) = 0$.
In what follows, we use $\top$ and $\bot$ to denote the concepts respectively given by
$\top(\vec x) = 1$ and $\bot(\vec x) = 0$ for all $\vec x \in \mathbb{B}^n$.

\medskip

Borrowing the terminology of computational learning theory  \cite{Kearns1994},
a \emph{representation class} (or \emph{language}) for a concept class $\mathcal F$ is a set of strings $\mathcal R$
defined over some (possibly infinite) alphabet of symbols. $\mathcal R$ is associated with two surjective functions, namely,
a mapping $\concept{\cdot} : \mathcal R \rightarrow \mathcal F$, called \emph{representation scheme},
and a mapping $\size{\cdot}: \mathcal R \rightarrow \mathbb N$, capturing the size of each representation.
Any string $\rho \in \mathcal R$ for which $\concept{\rho} = f$ is called a \emph{representation}
of the concept $f$.

A wide spectrum of representation classes have been proposed in the literature
for encoding Boolean functions in a compact way. Among them, \emph{propositional} languages are
defined over a set $X_n = \{x_1,\cdots, x_n\}$ of Boolean variables, the constants $1$ (true)
and $0$ (false), and the Boolean connectives $\neg$ (negation), $\lor$ (disjunction) and $\land$ (conjunction).
A \emph{literal} is a variable $x_i$ or its negation $\neg x_i$ (also denoted $\overline x_i$),
a \emph{term} or \emph{monomial} is a conjunction of literals, and a \emph{clause}
is a disjunction of literals. 
In such a setting, a vector $\vec x \in \mathbb{B}^n$ is also viewed as a term $\bigwedge_{i=1}^n \ell_i$,
where for each $i \in [n]$, $\ell_i = x_i$ if the $i$th coordinate of $\vec x$ is $1$, and $\ell_i = \overline{x_i}$ if the $i$th coordinate of $\vec x$ is $0$.
A \cnf\ formula is a finite conjunction of clauses. 

For propositional languages, the representation scheme is defined according to the standard semantics of propositional logic.
As an example, for the concept class of monomials, each representation is a term
$t = \ell_1 \land \cdots \land \ell_k$, and the corresponding concept is:
$$
t(\vec x) = \prod_{j = 1}^k \ell_j(\vec x)
$$
\begin{align*}
	\mbox{ where }
	\begin{cases}
		\ell_j(\vec x) = x_j     & \mbox{ if } \ell_j = x_j,           \\
		\ell_j(\vec x) = 1 - x_j & \mbox{ if } \ell_j = \overline x_j.
	\end{cases}
\end{align*}

As an alternative to propositional languages conveying a logical interpretation of Boolean functions,
\emph{neural} representation languages are endowed with a geometrical interpretation of concepts \cite{Anthony2001}.
The simplest neural representation language is the family of \emph{linear threshold functions} of the form
$f = (\vec w, \tau) \in \mathbb R^{n+1}$. For this language, the representation scheme maps $f$ to the concept:
\begin{align}
	\label{eq:LTU}
	f(\vec x) = \mathbb{1}[w_1x_1 + \cdots + w_nx_n \geq \tau]
\end{align}
where $\mathbb{1}[p] = 1$ if $p$ is true, and $\mathbb{1}[p] = 0$ if $p$ is false. 
These threshold units can be further
generalized to \emph{feedforward neural networks}, examined at the end of the next section. 

Let $\mathcal R$ be a representation language, and $\concept{\cdot}$ be a representation scheme mapping $\mathcal R$
into some concept class $\mathcal F$. For a representation $\rho \in \mathcal R$,
we use $\mathit{Var}(\rho)$ to denote the set of Boolean variables occurring in $\rho$.
The \emph{set} of models of $\rho$, given by $\concept{\rho}^{-1}(1)$ is denoted $\mathrm{mods}(\rho)$.
Whenever $\vec x$ belongs to $\mathrm{mods}(\rho)$, one also writes $\vec x \models \rho$.
Two representations $\rho$ and $\rho'$ of $\mathcal R$ are said to be \emph{equivalent}, denoted $\rho \equiv \rho'$, if $\concept{\rho} = \concept{\rho'}$.
We also say that $\rho$ \emph{entails} $\rho'$, denoted $\rho \models \rho'$, if $\mathrm{mods}(\rho) \subseteq \mathrm{mods}(\rho')$. 
A representation $\rho$ is \emph{inconsistent} if $\concept{\rho} = \bot$ and \emph{valid} if $\concept{\rho} = \top$.

\section{Boolean Classifiers}\label{sec:classifiers}

Based on elementary notions given in the previous section, we will focus on the concept class $\mathcal F = \mathcal F_n$.
In other words, all representation languages examined in this study are expressive enough to cover any Boolean function over $n$ variables.
In what follows, a \emph{Boolean classifier} is simply a representation of some concept in $\mathcal F_n$, according to some representation
language $\mathcal R$, associated with its representation scheme and its size measure.

For illustration, the following toy example will be used throughout the paper as a running example:

\begin{example}\label{running-ex}
The focus is laid on the concept of {\it common hollyhocks} (alias {\it alcea rosea}).
One needs a Boolean classifier to characterize it, i.e., to separate common hollyhocks from other roses using the following four features:
$x_1$: ``has a deciduous foliage'', $x_2$: ``has heart-shaped leaves'', $x_3$: ``has large flowers',  and $x_4$: ``has a light green stem''.
The concept $f \in \mathcal F_4$ of {\it common hollyhocks} is given by the set of its positive instances 
$\{(1, 0, 1, 1), (1, 1, 0, 0), (1, 1, 0, 1), (1, 1, 1, 0), (1, 1, 1, 1)\}$.
%

\end{example}


\paragraph{\dnf\ formulae.}
Arguably, the simplest language for representing in intuitive terms any Boolean function is
the class of \dnf\ formulae, which has been extensively studied in machine learning \cite{Valiant1985,Pitt1988,Feldman2009}.
A \dnf\ formula is a finite disjunction of monomials $D = t_1 \lor t_2 \lor \cdots \lor t_m$,
and its associated concept is $D(\vec x) = \max_{i = 1}^m t_i(\vec x)$.
As usual, the size of a \dnf\ formula is defined by the sum of sizes of its terms,
where the size of a term is simply given by the number of its literals.

\begin{example}
The concept of common hollyhocks can be represented by:
$$
\begin{array}{ll}
    D = &  \textcolor{white}{\vee}  
    	(x_1 \wedge x_2 \wedge \overline{x}_3) \vee  (x_1 \wedge x_2 \wedge x_3 \wedge x_4) \\
        &  \vee  (x_1 \wedge x_2 \wedge \overline{x}_4) \vee  (x_1 \wedge \overline{x}_2 \wedge x_3 \wedge x_4)
\end{array} 
$$
\end{example}

\paragraph{Decision Lists.}
The aforementioned \dnf\ formulae can be generalized to \emph{rule models}, which have received
a great deal of attention in the literature of machine learning and knowledge discovery (see e.g. \cite{Flach2012,Furnkranz2012} for general surveys).
Notably, \emph{decision lists} \cite{Rivest87} are ordered multi-sets of rules of the form
$L = \langle t_1, c_1\rangle, \ldots, \langle t_m, c_m\rangle$, where each $t_i$ ($i \in [m]$) is a term over $X_n$,
and each $c_i$ is a Boolean value in $\mathbb{B}$.
An input instance $\vec x \in \mathbb{B}^n$ is a model of $L$ if the class $c_i$ of the first rule $t_i$ that is matched on $\vec x$ is positive.
By convention, the last rule $t_m$ is the empty term $\top$. Formally, $L(\vec x) = c_j$ where $j = \argmin_{i = 1}^m \{t_i(\vec x) = 1\}$.
The size of a decision list $L$ is the sum of the sizes
of the terms occurring in $L$.

\begin{example}
The concept of common hollyhocks can be represented by
$L = \langle x_1 \wedge x_2, 1\rangle, \langle \overline{x}_1, 0\rangle, 
\langle x_3 \wedge x_4, 1\rangle, \langle \top, 0\rangle$.
\end{example}

\setcounter{figure}{0}
\begin{figure}[h]
  \centering
  \centering
  \scalebox{0.7}{     \begin{tikzpicture}[scale=1,
roundnode/.style={circle, draw=gray!60, fill=gray!5, very thick, minimum size=7mm},
squarednode/.style={rectangle, draw=red!60, fill=red!5, very thick, minimum size=5mm},
]
\node[roundnode] (n) at (1,0) {$x_2$};
\node[roundnode] (n1) at (-1,-1) {$x_1$};
\node[squarednode] (n11) at (-2,-2.5) {$0$};
\node[roundnode] (n12) at (0,-2.5) {$x_3$};
\node[squarednode] (n121) at (-1,-4) {$0$};
\node[roundnode] (n122) at (1,-4) {$x_4$};
\node[squarednode] (n1221) at (0,-5.5) {$0$};
\node[squarednode] (n1222) at (2,-5.5) {$1$};
\node[roundnode] (n2) at (3,-1) {$x_1$};
\node[squarednode] (n21) at (2,-2.5) {$0$};
\node[squarednode] (n22) at (4,-2.5) {$1$};

\draw[->,dashed] (n) -- (n1);
\draw[->] (n) -- (n2);

\draw[->,dashed] (n1) -- (n11);
\draw[->] (n1) -- (n12);

\draw[->,dashed] (n12) -- (n121);
\draw[->] (n12) -- (n122);

\draw[->,dashed] (n122) -- (n1221);
\draw[->] (n122) -- (n1222);

\draw[->,dashed] (n2) -- (n21);
\draw[->] (n2) -- (n22);
\end{tikzpicture}}
  \caption{A decision tree representation of the concept of common hollyhocks.}
  \label{fig:dt}  
\end{figure}
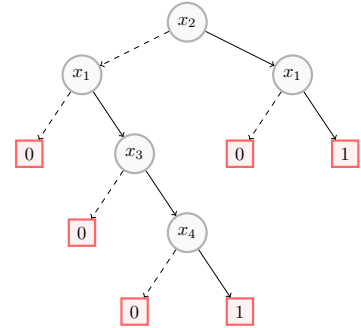
\setcounter{figure}{2}

\paragraph{Decision Trees.}
Tree models are among the most popular representations in machine learning.
In particular, \emph{decision trees} \cite{Breiman1984,Quinlan1986} are models of paramount importance
in XAI, as they can be easily read by recursively breaking a choice into sub-choice until a decision is reached.
Formally, a (Boolean) decision tree is a binary tree $T$, where each internal node is labeled with a Boolean variable
in $\mathcal X_n$, and each leaf is labeled $0$ or $1$. Without loss of generality, every variable is assumed to appear
at most once on any root-to-leaf path (this is called the \emph{read-once} property). The value $T(\vec x)$ of $T$ on
an input instance $\vec x \in \mathbb{B}^n$ is given by the leaf reached from the root as follows: for each internal node
labeled by $x_i$, go to the left or right child depending on whether the corresponding value $x_i$ of $\vec x$ is $0$ or $1$, respectively.
The size of $T$ is given by the number of its nodes.


\begin{example}
The concept of common hollyhocks can be represented by the decision tree $T$ in Figure \ref{fig:dt}.
\end{example}



\setcounter{figure}{1}
\begin{figure*}[t]
  \centering
  \scalebox{0.7}{
    \begin{tikzpicture}[scale=0.7,
roundnode/.style={circle, draw=gray!60, fill=gray!5, very thick, minimum size=7mm},
squarednode/.style={rectangle, draw=red!60, fill=red!5, very thick, minimum size=5mm},
]
\node[roundnode] (n) at (2,0.5) {$x_1$};
\node[squarednode] (n1) at (1,-1) {$0$};
\node[roundnode] (n2) at (3,-1) {$x_2$};
\node[squarednode] (n21) at (2,-2.5) {$0$};
\node[roundnode] (n22) at (4,-2.5) {$x_3$};
\node[squarednode] (n221) at (3,-4) {$1$};
\node[roundnode] (n222) at (5,-4) {$x_4$};
\node[squarednode] (n2221) at (4,-5.5) {$0$};
\node[squarednode] (n2222) at (6,-5.5) {$1$};

\draw[->,dashed] (n) -- (n1);
\draw[->] (n) -- (n2);
\draw[->,dashed] (n2) -- (n21);
\draw[->] (n2) -- (n22);
\draw[->,dashed] (n22) -- (n221);
\draw[->] (n22) -- (n222);
\draw[->,dashed] (n222) -- (n2221);
\draw[->] (n222) -- (n2222);

\node[roundnode] (m) at (13,0.5) {$x_2$};
\node[roundnode] (m1) at (9,-1) {$x_1$};
\node[squarednode] (m11) at (8,-2.5) {$0$};
\node[roundnode] (m12) at (10,-2.5) {$x_3$};
\node[squarednode] (m121) at (9,-4) {$0$};
\node[roundnode] (m122) at (11,-4) {$x_4$};
\node[squarednode] (m1221) at (10,-5.5) {$0$};
\node[squarednode] (m1222) at (12,-5.5) {$1$};
\node[roundnode] (m2) at (17,-1) {$x_3$};
\node[roundnode] (m21) at (15,-2.5) {$x_1$};
\node[roundnode] (m22) at (19,-2.5) {$x_1$};
\node[squarednode] (m211) at (14,-4) {$0$};
\node[squarednode] (m212) at (16,-4) {$1$};
\node[squarednode] (m221) at (18,-4) {$0$};
\node[squarednode] (m222) at (20,-4) {$1$};

\draw[->,dashed] (m) -- (m1);
\draw[->] (m) -- (m2);
\draw[->,dashed] (m1) -- (m11);
\draw[->] (m1) -- (m12);
\draw[->,dashed] (m12) -- (m121);
\draw[->] (m12) -- (m122);
\draw[->,dashed] (m122) -- (m1221);
\draw[->] (m122) -- (m1222);
\draw[->,dashed] (m2) -- (m21);
\draw[->] (m2) -- (m22);
\draw[->,dashed] (m21) -- (m211);
\draw[->] (m21) -- (m212);
\draw[->,dashed] (m22) -- (m221);
\draw[->] (m22) -- (m222);

\node[roundnode] (l) at (23,0.5) {$x_3$};
\node[squarednode] (l1) at (22,-1) {$0$};
\node[roundnode] (l2) at (24,-1) {$x_2$};
\node[roundnode] (l21) at (23,-2.5) {$x_4$};
\node[squarednode] (l22) at (25,-2.5) {$1$};
\node[squarednode] (l211) at (22,-4) {$0$};
\node[squarednode] (l212) at (24,-4) {$1$};

\draw[->,dashed] (l) -- (l1);
\draw[->] (l) -- (l2);
\draw[->,dashed] (l2) -- (l21);
\draw[->] (l2) -- (l22);
\draw[->,dashed] (l21) -- (l211);
\draw[->] (l21) -- (l212);

\end{tikzpicture}
  }
  \caption{A random forest representation of the concept of common hollyhocks.}
  \label{fig:rf}
\end{figure*}
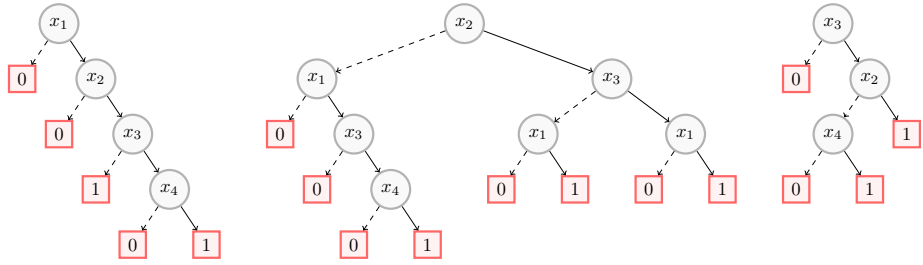

\paragraph{Random Forests.} Tree models can be generalized to \emph{tree ensembles}, using ensemble learning techniques, such as
bagging and boosting. Notably, the \emph{random forest} method generates multiple decision trees according to a variant of bagging
\cite{DBLP:journals/ml/Breiman96b,DBLP:journals/ml/Breiman01}. The output representation is a multi-set
$F = \{T_1,\cdots,T_m\}$
of decision trees, and the corresponding concept is given by:
\begin{align*}
	F(\vec x) =
	\begin{cases}
		1 & \mbox{ if } \sum_{i=1}^m T_i(\vec x) > \frac{m}{2} \\
		0 & \mbox{ otherwise.}
	\end{cases}
\end{align*}
In other words, an input instance $\vec x$ is a model of $F$ if and only if a strict majority of trees in $F$ classifies
$\vec x$ as a positive example. The size of $F$ is defined by the sum of sizes of the decision trees occurring in $F$.

\begin{example}
The concept of common hollyhocks can be represented by the random forest in Figure~\ref{fig:rf}.
\end{example}

\paragraph{Boosted Trees.}
Tree ensembles can also be trained using the boosting technique \cite{DBLP:journals/ml/Schapire90,DBLP:conf/eurocolt/FreundS95,SF2012} 
in order to yield \emph{boosted trees}, which are multi-sets of the form
$B = \{\langle T_1, \alpha_1\rangle, \ldots, \langle T_m, \alpha_m\rangle\}$,
where each $T_i \, (i \in [m])$ is a decision tree and $\vec \alpha$ is a convex combination of coefficients.\footnote{In other words, $\alpha_i \geq 0$ for all $i \in [m]$ and $\sum_i \alpha_i = 1$.}
By analogy with random forests, the decisions made by boosted trees are given from a weighted majority vote:
\begin{align*}
	B(\vec x) =
	\begin{cases}
		1 & \mbox{ if } \sum_{i=1}^m \alpha_i T_i(\vec x) > \frac{1}{2} \\
		0 & \mbox{ otherwise.}
	\end{cases}
\end{align*}
The size of the tree ensemble $B$ is the sum of the sizes of its trees.

\begin{example}
The concept of common hollyhocks can be represented by the boosted tree 
 in Figure \ref{fig:gbt}.  
\end{example}

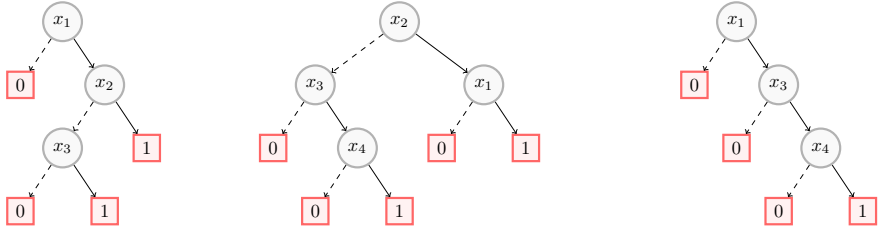
\begin{figure*}[t]
  \centering
  \scalebox{0.7}{\begin{tikzpicture}[scale=0.8,
roundnode/.style={circle, draw=gray!60, fill=gray!5, very thick, minimum size=7mm},
squarednode/.style={rectangle, draw=red!60, fill=red!5, very thick, minimum size=5mm},
]
\node[roundnode] (l) at (0,0.5) {$x_1$};
\node[squarednode] (l1) at (-1,-1) {$0$};
\node[roundnode] (l2) at (1,-1) {$x_2$};
\node[roundnode] (l21) at (0,-2.5) {$x_3$};
\node[squarednode] (l22) at (2,-2.5) {$1$};
\node[squarednode] (l211) at (-1,-4) {$0$};
\node[squarednode] (l212) at (1,-4) {$1$};

\draw[->,dashed] (l) -- (l1);
\draw[->] (l) -- (l2);
\draw[->,dashed] (l2) -- (l21);
\draw[->] (l2) -- (l22);
\draw[->,dashed] (l21) -- (l211);
\draw[->] (l21) -- (l212);

\node[roundnode] (n) at (8,0.5) {$x_2$};
\node[roundnode] (n1) at (6,-1) {$x_3$};
\node[roundnode] (n2) at (10,-1) {$x_1$};
\node[squarednode] (n11) at (5,-2.5) {$0$};
\node[roundnode] (n12) at (7,-2.5) {$x_4$};
\node[squarednode] (n121) at (6,-4) {$0$};
\node[squarednode] (n122) at (8,-4) {$1$};
\node[squarednode] (n21) at (9,-2.5) {$0$};
\node[squarednode] (n22) at (11,-2.5) {$1$};

\draw[->,dashed] (n) -- (n1);
\draw[->] (n) -- (n2);
\draw[->,dashed] (n1) -- (n11);
\draw[->] (n1) -- (n12);
\draw[->,dashed] (n12) -- (n121);
\draw[->] (n12) -- (n122);
\draw[->,dashed] (n2) -- (n21);
\draw[->] (n2) -- (n22);

\node[roundnode] (m) at (16,0.5) {$x_1$};
\node[squarednode] (m1) at (15,-1) {$0$};
\node[roundnode] (m2) at (17,-1) {$x_3$};
\node[squarednode] (m21) at (16,-2.5) {$0$};
\node[roundnode] (m22) at (18,-2.5) {$x_4$};
\node[squarednode] (m221) at (17,-4) {$0$};
\node[squarednode] (m222) at (19,-4) {$1$};

\draw[->,dashed] (m) -- (m1);
\draw[->] (m) -- (m2);
\draw[->,dashed] (m2) -- (m21);
\draw[->] (m2) -- (m22);
\draw[->,dashed] (m22) -- (m221);
\draw[->] (m22) -- (m222);

\end{tikzpicture}}
  \caption{\label{fig:gbt}A boosted tree representation of the concept of common hollyhocks. Weights of trees are respectively 0.5, 0.25 and 0.25.
  }
\end{figure*}

\paragraph{Boolean Multilayer Perceptrons.}
Based on the linear threshold units presented above, a neural network is formed when we place units 
at the vertices of a directed graph, with the arcs of the digraph describing the signal flows between units. 
More formally, a \emph{feedforward} neural network is defined by a directed acyclic graph $(V,E)$, and 
a weight function over the edges: $w: E \rightarrow \mathbb R$. Each node $v \in V$ of the graph captures a neuron.
In a \emph{multilayer} neural network, the set of nodes is decomposed into a union 
of (nonempty) disjoint subsets $V = \bigcupdot_{l=1}^d V_l$, 
such that every edge in $E$ connects every node in $V_{l}$ to every node in $V_{l+1}$, for some $l \in [d-1]$.    
Accordingly, every neuron $v \in V$ corresponds to a pair $l, i$ where $l \in [d]$ is a layer, and $i \in [\size{V_l}]$ is a rank in layer $l$.
The bottom layer $V_1$ is called the input layer and contains $n$ vertices.
The layers $V_2,\cdots,V_{d-1}$ are called \emph{hidden} layers, and the top layer $V_d$ is called the output layer.  
The inputs of  the $i$th neuron of the $l$th layer with $1 < l \leq d$ are the outputs of all the neurons
from layer $l-1$, plus an additional input $b_{l,i} \in \mathbb R$, called the bias.
We denote by $o_{l,i}(\vec x)$ the output of the $i$th neuron of the $l$th layer when the network 
is fed with the data instance $\vec x \in \mathbb R^n$. With this notation in hand, 
a multilayer neural network is recursively specified as follows:
\begin{align*}
o_{1,i}(\vec x) & = x_i \\
o_{l,i}(\vec x) & =  \mathit{sgn} \left( \sum_{j: (v_{l-1,j},v_{l,i}) \in E} w(v_{l-1,j},v_{t,i}) o_{l-1,j}(\vec x) + b_{l,i}\right)    
\end{align*}
\noindent where $\mathit{sgn}$ is the sign function such that $\mathit{sgn}(z) = \mathbb{1}[z \geq 0]$.
The depth, width, and size of the neural network are given by $d$, 
$\max_{l} \size{V_l}$, and $\size{V}$, respectively.
In a \emph{Boolean multilayer perceptron}, also known as \emph{Boolean multilayer threshold network} $P$
\cite{Anthony2001}, the input instances are vectors in the hypercube $\mathbb{B}^n$, 
and the output layer consists of a single neuron for which the output, denoted $P(\vec x)$, is a Boolean value in $\mathbb{B}$.

\begin{example}
The concept of common hollyhocks can be represented by the Boolean multilayer perceptron $P$ in Figure~\ref{fig:perceptron}. 
The weight of each edge is attached as a label to the corresponding edge. The bias associated with each neuron in layer $2$ is $-1$, and the bias
associated with the unique neuron in layer $3$ is $-3$. For the sake of readability, the corresponding inputs are not represented explicitly, but
the bias associated with a neuron is written in the box representing the neuron in the figure. 
\end{example}
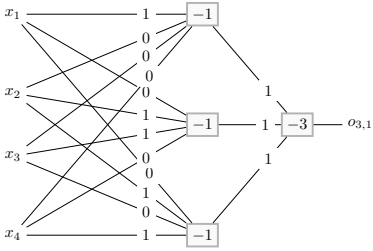
\begin{figure}[h]
  \centering
  \scalebox{0.6}{
    \begin{tikzpicture}[scale=0.7,
squarednode/.style={rectangle, draw=gray!60, fill=gray!5, very thick, minimum size=5mm},
weight/.style={near end,fill=white}
]
\node[squarednode] (c1) at (2,3.5) {$-1$};
\node[squarednode] (c2) at (2,0) {$-1$};
\node[squarednode] (c3) at (2,-3.5) {$-1$};
\node[squarednode] (co) at (5,0) {$-3$};
\node (o) at (7,0) {$o_{3,1}$};
\node (i1) at (-4,3.5) {$x_1$};
\node (i2) at (-4,1) {$x_2$};
\node (i3) at (-4,-1) {$x_3$};
\node (i4) at (-4,-3.5) {$x_4$};

\draw (i1) -- (c1) node [weight] {$1$};
\draw (i1) -- (c2) node [weight] {$0$};
\draw (i1) -- (c3) node [weight] {$0$};

\draw (i2) -- (c1) node [weight] {$0$};
\draw (i2) -- (c2) node [weight] {$1$};
\draw (i2) -- (c3) node [weight] {$1$};

\draw (i3) -- (c1) node [weight] {$0$};
\draw (i3) -- (c2) node [weight] {$1$};
\draw (i3) -- (c3) node [weight] {$0$};

\draw (i4) -- (c1) node [weight] {$0$};
\draw (i4) -- (c2) node [weight] {$0$};
\draw (i4) -- (c3) node [weight] {$1$};

\draw (c1) -- (co) node [weight] {$1$};
\draw (c2) -- (co) node [weight] {$1$};
\draw (c3) -- (co) node [weight] {$1$};

\draw (co) -- (o);

\end{tikzpicture}
  }
  \caption{A Boolean multilayer perceptron representation of the concept of common hollyhocks.}
  \label{fig:perceptron}
\end{figure}

\paragraph{Binarized Neural Networks.} 
Introduced in \cite{DBLP:conf/nips/HubaraCSEB16}, {\em binarized neural networks} are multilayer neural networks 
whose activations and weights are predominantly binary (but ranging in $\{-1,1\}$). 
A BNN is usually described in terms of composition of $d$ blocks of layers
(that are assembled sequentially) rather than individual layers. Thus, a BNN $N$ consists of a number (say, $m = d-1$) of internal blocks, followed by a unique output block, noted $O$.
Each block consists of a collection of linear and non-linear transformations. The $k$th internal block $BLK_k$ ($k \in [m]$) in a BNN can be modeled as a mapping
$$
	BLK_k: \{-1,1\}^{n_k} \rightarrow \{-1,1\}^{n_{k+1}} 
$$
associating with a vector of $n_k$ values in $\{-1,1\}$ a vector of $n_{k+1}$ values in $\{-1,1\}$.
The inputs of $BLK_1$ are the inputs of $N$ (thus, $n_1 = n$, the number of elements in $X_n$),
the outputs of $BLK_k$ ($k \in [m-1]$) are the inputs of $BLK_{k+1}$,
the output of $BLK_m$ is the input of $O$, and the output value of $O$ is the output of $N$.
While the input and output of $N$ are binary vectors, the internal layers of each internal block
can produce real-valued intermediate outputs. A common construction of an internal block $BLK_k$ ($k \in [m]$)
is composed of three main operations:
\begin{enumerate}
	\item linear transformation (LIN):
	      $$
		      \vec{y} = \vec{A}_k\vec{x}_k + \vec{b}_k,$$       
		$\text{where } \vec{A}_k \in \{-1,1\}^{n_{k+1} \times n_k}
		\text{ and }
		      \vec{b}_k \in  \mathbb{R}^{n_{k+1}}$

	\item batch normalization (BN):
	      $$
		      z_i = \alpha_{k_i} (\frac{y_i - \mu_{k_i}}{\nu_{k_i}}) + \gamma_{k_i},
	      $$
	     $ \text{ where }
		      \vec{y} = (y_1, y_2, \ldots, y_{n_{k+1}}), \alpha_{k_i}, \mu_{k_i}, \nu_{k_i}, \gamma_{k_i} \in \mathbb{R},\\
		\text{ and }\nu_{k_i} > 0
	      $
	\item binarization (BIN):
	      $$
		      \vec{x}_{k+1} = \mathit{sgn}(\vec{z}),
	      $$
	      $
		\text{ where } 	\vec{z} = (z_1, z_2, \ldots, z_{n_{k+1}}) \in \mathbb{R}^{n_{k+1}},\\ \text{ and for each } i \in [n_{k+1}], 
		       \mathit{sgn}(z_i) = 1 \text{ if } z_i \geq 0\\
		        \text{ and } \mathit{sgn}(z_i) = -1 \text{ if } z_i < 0, \\
			 \text{ so that } \vec{x}_{k+1} \in \{-1,1\}^{n_{k+1}}
	      $
\end{enumerate}

The output block produces the classification decision. It consists of two layers:
\begin{enumerate}
	\item linear transformation (LIN):
	      $$
		     \vec{w} = \vec{A}_d\vec{x}_d + \vec{b}_d, \text{ where } \vec{A}_d \in \{-1, 1\}^{s \times n_d}
		      \text{ and } \vec{b}_d \in \mathbb{R}^s
	      $$
	\item argmax layer (ARGMAX), which outputs the largest index of the largest entry in $\vec{w}$ as the predicted label
	      $$
		      o = \mathrm{argmax}_{i=1}^s \{w_i\} 
	      $$
	      Since we are interested in Boolean classification, we suppose that the number of output values of $N$ is $s = 2$ and that for any $\vec{x} \in \{-1, 1\}^n$,
	      $N$ classifies $\vec{x}$ as a positive instance if and only if $w_2 > w_1$ (the value of $o$ is $2$ in this case, and $1$ otherwise).
\end{enumerate}

\begin{example}
  The concept of common hollyhocks can be represented by the following BNN $N$, with $d = 2$ blocks.
  We consider only one internal block, with four inputs and five outputs. The parameters of $N$ are defined as follows:

LIN:
\[
\vec{A}_1 = \begin{pmatrix*}[r]
    1  & -1 &  1 &  1 \\ 
    1  &  1 & -1 & -1 \\ 
    1  &  1 & -1 &  1 \\ 
    1  &  1 &  1 & -1 \\ 
    1  &  1 &  1 &  1 \\ 
\end{pmatrix*},
\]
\[
\vec{b}_1 = (-3.5,-3.5,-3.5,-3.5, -3.5)
\]

BN:
\[
\mathbf{\alpha_1} = \mathbf{\nu_1} = (1,1,1,1,1), \mathbf{\gamma_1} = \mathbf{\mu_1} = (0,0,0,0,0)
\]

The output block $O$ is defined by:

LIN:
\[
\vec{A}_d = \begin{pmatrix*}[r]
   -1  & -1 & -1 & -1 & -1\\ 
    1  &  1 &  1 &  1 &  1\\ 
\end{pmatrix*},
\vec{b}_d = (-4.5,5)
\]

\end{example}

\section{XAI Queries as Computation Problems}\label{sec:XAIqueries}


In this section, we consider successively the $9$ XAI queries from \cite{Audemardetal20}, as listed in Table \ref{table:summary}, and we present them in formal terms. 

\paragraph{EMC: Enumerating Minimum-Cardinality explanations}
Given an input $\vec x$ such that $\rho(\vec x) = c$, a \emph{minimum-cardinality explanation} \cite{ShihCD18} of $\vec x$
is an instance $\vec x'$ such that $\rho(\vec x') = c$,
$\vec x'$ coheres with  $\vec x$ on the ones in the sense that for any $k \in \{1, \ldots, n\}$, if $x'_k = 1$ then
$x_k = 1$, and $\vec x'$ has a minimal number of coordinates set to $1$. Roughly speaking, the features that are set to $1$ in $\vec x'$ are enough to explain why
$\vec x$ has been classified by $\rho$ as a positive (or as a negative) instance. Formally:


\begin{definition}[EMC]
	EMC can be stated as the following problem:
	\begin{itemize}
		\item \underline{Input:} A Boolean representation $\rho$ over $X_n$ and an instance $\vec x \in \mathbb{B}^n$.
		\item \underline{Output}: Enumerate with polynomial delay the set of all minimum-cardinality explanations of $\vec x$ given $\rho$.
	\end{itemize}
\end{definition}

The number of minimum-cardinality explanations of $\vec x$ given $\rho$ can be exponential in the size of the input, thus
the time needed to compute all of them is provably exponential as well. 
EMC[1] denotes the relaxation of EMC where the output consists of a single minimum-cardinality explanation of $\vec x$ given $\rho$. 
Considering any Boolean classifier $\rho$ for Example \ref{running-ex}, $(1, 1, 0, 0)$ is the output of EMC[1] for input $\rho$ and $\vec x = (1, 1, 1, 1)$.


%


\paragraph{DPI: Deriving one Prime Implicant explanation}
Given an input $\vec x$ such that $\rho(\vec x) = c$, a \emph{prime implicant explanation} of $\vec x$ \cite{ShihCD18}
(also referred to as a sufficient reason for $\vec x$ given $\rho$ \cite{DarwicheHirth20})  is a subset-minimal partial assignment $\vec x'$ which is coherent with 
$\vec x$ (i.e., $\vec x$  and $\vec x'$ give the same values to the variables that are assigned in $\vec x'$) and which satisfies the property that for every extension $\vec x''$ of $\vec x'$ over $X_n$, we have $\rho(\vec x'') = c$. The features assigned in $\vec x'$ (and the way they are assigned) can be viewed as explaining why $\vec x$ has been classified by $\rho$ as a positive (or as a negative) instance. Formally: 

\begin{definition}[DPI]
	DPI can be stated as the following problem:
	\begin{itemize}
		\item \underline{Input:} A Boolean representation $\rho$ over $X_n$ and  an instance $\vec x \in \mathbb{B}^n$.
		\item \underline{Output}: A prime implicant explanation of $\vec x$ given $\rho$.
	\end{itemize}
\end{definition}

Considering any Boolean classifier $\rho$ for Example \ref{running-ex}, $x_1 \wedge x_2$ can be got as an output of DPI for input $\rho$ and $\vec x = (1, 1, 1, 1)$.
%
%
%


\paragraph{ECO:  Enumerating COunterfactual explanations}
Counterfactual explanations are required when the user is surprised by the result $\vec y$ provided by the classifier $\rho$ on a given instance $\vec x$.
We have $\rho(\vec x) = 1$ (resp. $=0$) while the user was expecting $\rho(\vec x) = 0$ (resp. $=1$).
A \emph{counterfactual explanation} of $\vec x$ given $\rho$ is an instance $\vec x'$ which is as close as possible to
$\vec x$ in terms of Hamming distance and such that $\rho(\vec x') \neq \rho(\vec x)$.
If there is no $\vec x'$ such that $\rho(\vec x') \neq \rho(\vec x)$,
then no counterfactual explanation of $\vec x$ given $\rho$ exists. When $\vec x'$ exists, the set of features that differ in $\vec x$ and $\vec x'$
can be viewed as an explanation as to why $\vec x$ has not been classified as expected by $\rho$. Formally: 

\begin{definition}[ECO]
	ECO can be stated as the following problem:
	\begin{itemize}
		\item \underline{Input:} A Boolean representation $\rho$ over $X_n$ and an instance $\vec x \in \mathbb{B}^n$.
		\item \underline{Output}: Enumerate with polynomial delay the set of all counterfactual explanations of $\vec x$ given $\rho$.
	\end{itemize}
\end{definition}

The number of counterfactual explanations of $\vec x$ given $\rho$ can be exponential in the size of the input, thus
the time needed to compute all of them is provably exponential as well. 
ECO[1] denotes the relaxation of ECO where the output consists of a single
counterfactual explanation of $\vec x$ given $\rho$ when such an explanation exists, and $\emptyset$ otherwise.
Considering any Boolean classifier $\rho$ for Example \ref{running-ex}, $(0, 1, 1, 1)$ is the output of ECO[1] 
for input $\rho$ and $\vec x = (1, 1, 1, 1)$.

%


\paragraph{CIN:  Counting the INstances associated with a given class}

Counting the number of instances associated with the given class corresponding to $\rho$
is a useful verification query. When the number found heavily differs from the expected one, this
may reflect an issue with the dataset used to learn the parameters of the classifier. Formally: 

\begin{definition}[CIN]
	CIN can be stated as the following problem:
	\begin{itemize}
		\item \underline{Input:} A Boolean representation $\rho$ over $X_n$, and a target class $c \in \mathbb{B}$ (positive or negative instances).
		\item \underline{Output}: The number of instances $\vec x \in \mathbb{B}^n$ classified by $\rho$ as positive instances if $c = 1$, or as negative instances if $c = 0$.
	\end{itemize}
\end{definition}

Considering any Boolean classifier $\rho$ for Example \ref{running-ex}, $11$ is the output of CIN
for input $\rho$ and $c=0$.
%
%


\paragraph{EIN:  Enumerating the INstances associated with a given class}

EIN is the enumeration problem that corresponds to CIN:

\begin{definition}[EIN]
	EIN can be stated as the following problem:
	\begin{itemize}
		\item \underline{Input:} A Boolean representation $\rho$ over $X_n$, and a target class $c \in \mathbb{B}$ (positive or negative instances).
		\item \underline{Output}: Enumerate with polynomial delay the set of positive instances $\vec x \in \mathbb{B}^n$ according to $\rho$ if $c = 1$ and the set of negative instances 
		$\vec x \in \mathbb{B}^n$
		      according to $\rho$ if $c = 0$.
	\end{itemize}
\end{definition}

The number of positive (or negative) instances $\vec x \in \mathbb{B}^n$ according to $\rho$ can be exponential in the size of the input, thus
the time needed to compute all of them is provably exponential as well. 
EIN[1] denotes the relaxation of EIN where the output consists of a single instance.
Considering any Boolean classifier $\rho$ for Example \ref{running-ex}, $(0, 1, 1, 1)$ can be got as an output of EIN[1]
for input $\rho$ and $c=0$.

%


\paragraph{IMA:  Identifying MAndatory features / forbidden features in a given class}

When the frequency of a feature $x_k$ (or combination of features) in the class of positive (or negative) instances associated with $\rho$ is equal to $1$,
the feature / combination of features is {\em mandatory} for an instance to be recognized as an element of the class,
while when it is equal to $0$, it is {\em forbidden}. 
Identifying the mandatory and forbidden features for the classes of positive
(or negative) instances (as they are perceived by the classifier)
is useful (the classifier should be such that there is no discrepancy between what is got and what was expected). Formally: 

\begin{definition}[IMA]
	IMA can be stated as the following problem:
	\begin{itemize}
		\item \underline{Input:} A Boolean representation $\rho$ over $X_n$, a term $t$ over $X_n$, and a target class $c \in \mathbb{B}$ (positive or negative instances).
		\item \underline{Output}: $1$ if $t$ is mandatory for the class of positive (resp. negative) instances when $c = 1$ (resp. $c = 0$), and $0$ otherwise.
	\end{itemize}
\end{definition}

A similar definition can be stated for forbidden features. Considering any Boolean classifier $\rho$ for Example \ref{running-ex}, $1$ is the output of IMA
for input $\rho$, $t = x_1$, and $c=1$.
%
%



\paragraph{IIR:  Identifying IRrelevant features in a given class}

A feature $x_i$ is \emph{irrelevant} for the class of positive (resp. negative) instances associated with $\rho$ if and only if for every positive
(resp. negative) instance $\vec x$ according to $\rho$, the instance $\vec x'$ that coincides with $\vec x$ on every feature but $x_i$ is also classified positively
(resp. negatively) by $\rho$.
Deciding whether a feature is irrelevant or not for the class of positive (resp. negative) instances associated with $\rho$ is a useful verification query for identifying decision bias:
there is such a bias when the membership of any instance $\vec x$ to the class associated with $\rho$ depends on its value for the feature $x_i$ while it should not. Formally:

\begin{definition}[IIR]
	IIR can be stated as the following problem:
	\begin{itemize}
		\item \underline{Input:} A Boolean representation $\rho$ over $X_n$, a feature $x_i \in X_n$, and a target class $c \in \mathbb{B}$ (positive or negative instances).
		\item \underline{Output}: $1$ if $x_i$ is irrelevant for the class of positive (resp. negative) instances associated with $\rho$ when $c = 1$ (resp. $c = 0$), and $0$ otherwise.
	\end{itemize}
\end{definition}

Considering any Boolean classifier $\rho$ for Example \ref{running-ex}, $0$ is the output of IIR
for input $\rho$ and $c=1$, whatever $x_i$ ($i \in [4]$).

%


\paragraph{IMO:  Identifying MOnotone (or anti-monotone) features in a given class}

In many applications, it is believed that increasing the value of some feature does not change the
membership to the class of positive (resp. negative) instances associated with the Boolean classifier.
Dually, one might also expect that decreasing the value of some other feature
does not change the membership to the class. It is important to be able to test whether the classifier $\rho$
that has been generated complies or not with such beliefs.

Making it formal calls for a notion of monotonicity (or anti-monoto\-ni\-ci\-ty) of a classifier, which can be stated as follows:
a classifier $\rho$ is \emph{monotone} (resp. \emph{anti-monotone}) with respect to an input feature $x_i$ for the class of positive (resp. negative) instances, if for any positive (resp. negative) instance $\vec x$ according to $\rho$, we have $\rho(\vec x[x_i \leftarrow 1]) = 1$ (resp. $\rho(\vec x[x_i \leftarrow 0]) = 1)$.\footnote{If $\vec x = (x_1, \cdots, x_n)$, then $\vec x[x_k \leftarrow v]$ is the same vector as $\vec x$, except that the $j$th coordinate $x_k$ of $\vec x[x_k \leftarrow v]$ has value $v$.} Formally: 

\begin{definition}[IMO]
	IMO can be stated as the following problem:
	\begin{itemize}
		\item \underline{Input:} A Boolean representation $\rho$ over $X_n$, a feature $x_i \in X_n$, and a target class $c \in \mathbb{B}$ (positive or negative instances).
		\item \underline{Output}: $1$ if $\rho$ is monotone (resp. anti-monotone) w.r.t. $x_i$ for the class of positive (resp. negative) instances, $0$ otherwise.
	\end{itemize}
\end{definition}

Considering any Boolean classifier $\rho$ for Example \ref{running-ex}, $1$ is the output of IMO 
for input $\rho$ and $c=1$, whatever $x_i$ ($i \in [4]$).

%


\paragraph{MCP:  Measuring Closeness of a class to a Prototype}


Finally, one can also be interested in determining how much a given prototype $\vec x$ complies with the class that the classifier $\rho$ associates with it. 
This can be evaluated by computing the Hamming distance between $\vec x$ and every element of $\{\vec x' \in \mathbb{B}^n : \rho(\vec x') = \rho(\vec x) \}$
and considering the maximal distance. 
When a prototype of a class exists, it is supposed to be
a ``central''  element of the class (i.e., minimizing the maximal distance to any other element of the class). Thus, a large value
may indicate a problem with the classifier that has been learned. Formally: 

\begin{definition}[MCP]
	MCP can be stated as the following problem:
	\begin{itemize}
		\item \underline{Input:} A Boolean representation $\rho$ over $X_n$ and
		      an instance $\vec x \in \mathbb{B}^n$.
		\item \underline{Output}: The maximal Hamming distance of $\vec x$ to the class of positive (resp. negative) instances 
		      when $\rho(\vec x) = 1$ (resp. $\rho(\vec x) = 0$).
	\end{itemize}
\end{definition}

Considering any Boolean classifier $\rho$ for Example \ref{running-ex}, $2$ is the output of MCP 
for input $\rho$ and $\vec x = (1, 1, 1, 1)$.

%

\section{On the Intelligibility of XAI Queries}\label{sec:complexity}

We are now in position to evaluate the computational intelligibility of each family of classifiers, among decision trees, \dnf\ classifiers, decision lists, random forests, boosted
trees, Boolean multilayer perceptrons, and binarized neural nets over Boolean features. This intelligibility is assessed by determining the set of XAI queries 
(out of the $9$ ones considered in the previous section) that are offered by each family, i.e., those for which the corresponding computation problem is tractable.

Since the computation problems associated with XAI queries are not always decision problems, the intractability of a computation problem is established by 
proving that it is {\sf NP}-hard in the sense of Cook reduction; in this case, the existence of a (deterministic) polynomial-time algorithm to solve the corresponding XAI 
query would imply that {\sf P = NP}, giving thus strong evidence that such an algorithm does not exist.


The main results of the paper are synthesized in the two following propositions: 

\begin{proposition}\label{prop:complexDT}
	For each enumeration problem among EMC, ECO, EIN, there exists an enumeration algorithm with polynomial delay when the Boolean classifier under consideration
	is a decision tree over $X_n$. Furthermore, each problem among DPI, CIN, IMA, IIR, IMO, MCP  is in {\sf P} when the Boolean classifier under consideration is a decision tree over $X_n$.
\end{proposition}

\begin{proof}
	By definition, for each of the $9$ XAI queries, the target class can be the one of positive instances or the one of negative instances.
	This does not raise any issue for decision trees. Indeed, for any decision tree $T$ over $X_n$, one can compute in linear time a decision tree $T'$ representing the complementary class
	to the one associated with $T$, i.e., a decision tree $T'$ such that $\forall \vec x \in \mathbb{B}^n$, $T'(\vec x) = 1$ if and only if $T(\vec x) = 0$.
	To get $T'$ from $T$, it it enough to replace in $T$ every $1$-leaf node by a $0$-leaf node, and  every $0$-leaf node by a $1$-leaf node.\footnote{Stated
		otherwise, \dt\ satisfies the $\nc$ transformation from the knowledge compilation map \cite{Koricheetal13}.}

	Now, \cite{Audemardetal20} have identified sufficient conditions for a (multi-label, yet Boolean) classifier to offer XAI queries
	based on the queries and transformations of the language \L\ used to represent it. Those queries and transformations are standard queries and
	transformations from the knowledge compilation map \cite{DarwicheMarquis02}. 
%

	It turns out that the language \dt\ of decision trees over Boolean variables satisfies many of those queries and transformations, namely
	\co, \cd, \me, \ct, \im, \opt, \eq, \se. This has been shown in \cite{Koricheetal13} for all of them, but \opt.
	As to \opt, it is easy to adapt the proof that \Dnnf\ satisfies \opt\  \cite{DarwicheMarquis04,Koricheetal16}
	to the case of \dt. Indeed, let $w_{v} \in \mathbb{Q}$ be a number (the weight of $v \in X_n$). For any interpretation
	$\vec x$ over $X_n$, one defines $f_w(\vec x) = \sum_{v \in X_n} w_v \cdot \vec x(v)$. Now, for any formula $\varphi$,
	one defines $f_w(\varphi) = \mathit{min}(\{f_w(\vec x) : \vec x \models \varphi\})$.\footnote{We set $w_v$ to $0$ whenever $v$ does not occur in $\varphi$.}
	It is easy to show by structural induction that when $\alpha = T$ is a decision tree over $X_n$, $f_w(T)$ can be computed in time polynomial
	in the size of $T$ when all the weights $w_{v}$ are bounded by a constant that does not depend on $T$ (which is a reasonable assumption).
	Indeed, we have $f_w(0) = \infty$, $f_w(1) = 0$, and 
	$$f_w(\mathit{ite}(v, T_1, T_2)) = w_v + \mathit{min}(\{f_w(T_1), f_w(T_2)\}).$$

	On this ground, starting from $T$, one can generate in polynomial time a decision tree
	$\mathit{opt}(T)$ over $\mathit{Var}(T)$ the models of which being precisely the
	models of $T$ over $\mathit{Var}(T)$, that minimize the value of $f_w$. Indeed, we have $\mathit{opt}(0) = 0$, $\mathit{opt}(1) = 1$, and\\
	~\\
	$\mathit{opt}(\mathit{ite}(v, T_1, T_2))$ \\
		      \begin{tabular}[t]{ll}
			      $= \mathit{ite}(v, \mathit{opt}(T_1), \mathit{opt}(T_2))$ & $\mbox{if } f_w(T_1) = f_w(T_2)$ \\
			      $= \mathit{ite}(v, \mathit{opt}(T_1), 0)$                 & $\mbox{if } f_w(T_1) < f_w(T_2)$ \\
			      $= \mathit{ite}(v, 0, \mathit{opt}(T_2))$                 & $\mbox{if } f_w(T_1) > f_w(T_2)$ \\
		      \end{tabular}
	~\\

	Then using results reported in \cite{Audemardetal20}, we get that \dt\ offers the XAI queries EMC, ECO, CIN, EIN, IMA, MCP.
	Finally, though \dt\ does not satisfy \fo\ \cite{Koricheetal13}, the XAI queries that have been addressed using forgetting in \cite{Audemardetal20} require
	to apply the forgetting transformation to eliminate from $T$ variables representing classes. This is useless here since no class variable is used
	in $T$ (only two classes are implicitly considered here, the one associated with $T$, alias the class of positive instances, and its complementary set
	which can be obtained by computing $T'$). 
	Thus, \dt\ also offers the XAI queries DPI,\footnote{A more direct proof can be found in \cite{DBLP:journals/corr/abs-2010-11034}.} IIR, and IMO.

\end{proof}


\begin{proposition}\label{prop:complexothers}
	Each problem among EMC[1], DPI, ECO[1], CIN, EIN[1], IMA, IIR, IMO, MCP  is {\sf NP}-hard when the Boolean classifier under consideration is a \dnf\ formula, a decision list, a random forest, a boosted tree, a Boolean multilayer perceptron, or a binarized neural network over $X_n$.
\end{proposition}

%
%
%
%
%

\begin{proof}
The proof is organized into three parts. In a first part, we show that the well-known SAT problem for \cnf\ formulae can be reduced in polynomial time to every problem among EMC[1], DPI, ECO[1], CIN, EIN[1], IMA, IIR, IMO, MCP where the Boolean classifier under consideration $\rho$ is given as a \cnf\ formula. 

In a second part, we show how a \cnf\ classifier $\rho$ can be associated in polynomial time with an equivalent classifier having the form of a decision list, a random forest, 
a boosted tree, a Boolean multilayer perceptron, or a binarized neural net over Boolean features. 

Combing the polynomial reductions from the first part with the polynomial translations of the second part, the {\sf NP}-hardness results stated in the proposition and concerning
decision lists, random forests, boosted trees, Boolean multilayer perceptrons, and binarized neural nets follow.
Finally, the case of \dnf\ classifiers is addressed in a third part.

Let us start with the first part of the proof:
\begin{itemize}
\item {\bf EMC[1].}~ 
	Let $\alpha = \bigwedge_{i=1}^k \delta_i$ be a \cnf\ formula over $\{x_1, \ldots, x_{n-1}\}$.
	We associate with $\alpha$ in polynomial time the ordered pair $(\rho, \vec x)$ where 
	$\rho =  \bigwedge_{i=1}^k \bigwedge_{j=1}^n (\delta_i \vee x_j)$ is a \cnf\ formula over $X_n = \{x_1, \ldots, x_n\}$ (equivalent to $\alpha \vee (\bigwedge_{i=1}^n x_i)$), and 
	$\vec x = \bigwedge_{i=1}^n x_i$. 
	Clearly enough, $\rho$ classifies $\vec x$ as a positive instance. Now, there are two cases:
	\begin{itemize}
		\item If $\alpha$ is unsatisfiable, then $\rho$ is equivalent to $\bigwedge_{i=1}^n x_i$. In this case, the sole minimum-cardinality explanation of $\vec x$ given $\rho$ is equal to 
		$\vec x$.
		\item If $\alpha$ is satisfiable, then it has a model $\vec x'$ over $\{x_1, \ldots, x_{n-1}\}$. The instance $\vec x'' \in \mathbb{B}^n$ that extends $\vec x'$ and sets $x_n$ to $0$
		      is classified as a positive instance by $\rho$, and it contains less features set to $1$ than $\vec x$, thus $\vec x$ is not a minimum-cardinality explanation of $\vec x$ in this case.
	\end{itemize}

\item {\bf DPI.}~ 
	Let $\alpha$ be a \cnf\ formula over $\{x_1, \ldots, x_{n-1}\}$. 
	We associate with $\alpha$ in polynomial time the ordered pair $(\rho, \vec x)$
	where $\rho = \alpha \wedge (\bigvee_{i=1}^n \overline{x_i})$ is a \cnf\ formula over $X_n = \{x_1, \ldots, x_n\}$ and $\vec x = \bigwedge_{i=1}^n x_i$.
	By construction, $\vec x$ is classified by $\rho$ as a negative instance. Now:
	\begin{itemize}
		\item If $\alpha$ is unsatisfiable, then every instance $\vec x' \in \mathbb{B}^n$ is classified by $\rho$ as a negative instance since $\rho$ is equivalent to $\bot$.
		      This is equivalent to state that there is a unique prime implicant explanation of $\vec x$ classified by $\rho$ as a negative instance, namely $\top$.
		\item If $\alpha$ is satisfiable, then it has a model over $\{x_1, \ldots, x_{n-1}\}$, and the instance $\vec x' \in \mathbb{B}^n$ that extends this model and sets $x_n$ to $0$ is a model
		      of $\alpha \wedge (\bigvee_{i=1}^n \overline{x_i})$. Thus, $\vec x'$ is classified by $\rho$ as a positive instance, and as a consequence
		      $\top$ is not a prime implicant explanation of $\vec x$ given $\rho$.
	\end{itemize}

\item {\bf ECO[1].}~ 
	Let $\alpha$ be a \cnf\ formula over $\{x_1, \ldots, x_{n-1}\}$. Consider the same polynomial reduction as in the DPI case. There are two cases: 
	\begin{itemize}
		\item If $\alpha$ is unsatisfiable, then every instance $\vec x' \in \mathbb{B}^n$ is classified by $\rho$ as a negative instance since $\rho$ is equivalent to $\bot$.
		      Thus, in this case, there is no counterfactual explanation of $\vec x$ given $\rho$.
		\item If $\alpha$ is satisfiable, then it has a model over $\{x_1, \ldots, x_{n-1}\}$.
		      The instance $\vec x' \in \mathbb{B}^n$ that extends this model and sets $x_n$ to $0$ is a model of $\alpha \wedge (\bigvee_{i=1}^n \overline{x_i})$. In this case, the
		      set of positive instances given $\rho$ is not empty, and as a consequence, a counterfactual explanation of $\vec x$ given $\rho$ exists.
	\end{itemize}

\item {\bf CIN.}~ 
	Let $\alpha$ be a \cnf\ formula. 	We associate with $\alpha$ in polynomial time the ordered pair $(\rho, c)$
	where $\rho = \alpha$, and $c = 1$.
	The number of instances $\vec x \in \mathbb{B}^n$ classified positively by $\rho$ is equal to the number of models of $\alpha$. If it was
	possible to compute this number in polynomial time, then one could decide in polynomial time whether $\alpha$ is satisfiable or not.

\item {\bf EIN[1].}~ 
	We consider the same polynomial reduction 
	as 
	in the CIN case. 
	An instance $\vec x \in \mathbb{B}^n$ classified by $\rho$ as a positive instance exists if and only if $\alpha$ is satisfiable.

\item {\bf IMA.}~ 
	Let $\alpha$ be a \cnf\ formula over $\{x_1, \ldots, x_{n-1}\}$. We associate with $\alpha$ in polynomial time the triple $(\rho, t, c)$ where
	$\rho$ is the same formula as in the proof for the EMC[1] case, $t = x_n$, and $c = 1$:
	\begin{itemize}
		\item If $\alpha$ is unsatisfiable, then $\vec x = \bigwedge_{i=1}^n x_i$ is the unique instance of $\mathbb{B}^n$ that is classified positively by $\rho$.
		      Thus, every feature from $\vec x$ (especially, those of $t = x_n$) is mandatory for the class of positive instances.
		\item If $\alpha$ is satisfiable, then it has a model over $\{x_1, \ldots, x_{n-1}\}$ and the instance $\vec x'$ that extends this model
		      and sets $x_n$ to $0$ is classified positively by $\rho$, showing
		      that $t = x_n$ is not mandatory for the class of positive instances.
	\end{itemize}

	The case of forbidden features is similar (consider $\rho = \alpha \vee (\bigwedge_{i=1}^n \overline{x_i})$ instead of
	$\rho = \alpha \vee (\bigwedge_{i=1}^n x_i)$: $t$ is forbidden for the class of positive instances if and only if $\alpha$ is unsatisfiable).

\item {\bf IIR.}~ 
	Let $\alpha$ be a \cnf\ formula over $\{x_1, \ldots, x_{n-1}\}$. 
	Let us associate with $\alpha$ in polynomial time the triple $(\rho, x_n, c)$ where
	$\rho = \alpha \wedge x_n$ is a \cnf\ formula over $X_n = \{x_1, \ldots, x_n\}$, and $c = 0$:
	\begin{itemize}
		\item If $\alpha$ is unsatisfiable, then $\rho$ is unsatisfiable as well, and $x_n$ is an irrelevant feature for the class of negative instances associated with $\rho$.
		\item If $\alpha$ is satisfiable, then it has a model over $\{x_1, \ldots, x_{n-1}\}$. Consider the instance
		      $\vec x$ that extends this model and sets $x_n$ to $1$. Then $\rho(\vec x) = 1$. However the instance $\vec x' = \vec x[x_n \leftarrow 0]$
		      that coincides with $\vec x$ on every feature but $x_n$ is such that $\rho(\vec x') = 0$. This shows that $x_n$ is relevant for the class of negative instances
		      associated with $\rho$.
	\end{itemize}

\item {\bf IMO.}~
	Let $\alpha$ be a \cnf\ formula over $\{x_1, \ldots, x_{n-1}\}$. 
	Let us associate with $\alpha$ in polynomial time the triple $(\rho, x_n, c)$ where $\rho = \alpha \wedge \overline{x_n}$ is a \cnf\ formula 
	over $X_n = \{x_1, \ldots, x_n\}$, and $c = 1$:
	\begin{itemize}
		\item If $\alpha$ is unsatisfiable, then $\rho$ is unsatisfiable as well and, as such, $\rho$ is obviously monotone w.r.t. $x_n$ (it is monotone w.r.t. every feature).
		\item If $\alpha$ is satisfiable, then it has a model over $\{x_1, \ldots, x_{n-1}\}$. Consider the instance
		      $\vec x$ that extends this model and sets $x_n$ to $0$. Then $\rho(\vec x) = 1$. However the
		      instance $\vec x' = \vec x[x_n \leftarrow 1]$ that coincides with $\vec x$ on every feature but $x_n$ is such that $\rho(\vec x') = 0$.
		      This shows that $\rho$ is not monotone w.r.t. $x_n$ for the class of positive instances.
	\end{itemize}
	The case of anti-monotone features is similar
	(consider $\rho = \alpha \wedge x_n$ and extends the counter-model of $\alpha$ by setting $x_n$ to $1$ when $\alpha$ is satisfiable).

\item {\bf MCP.}~
  Let $\alpha$ be a \cnf\ formula over $\{x_1, \ldots, x_{n-1}\}$. Consider the same polynomial reduction as in the EMC[1] case:
	\begin{itemize}
		\item If $\alpha$ is unsatisfiable, then the unique instance classified positively by $\rho$ is $\vec x$,
		      showing that the maximal Hamming distance of $\vec x$ to an element of the class associated with $\rho$ is $0$.
		\item If $\alpha$ is satisfiable, then it has a model over $\{x_1, \ldots, x_{n-1}\}$ and the instance $\vec x'$ that extends this model and
		      is such that $x_n$ is set to $0$ is classified positively by $\rho$.
		      In this case, the maximal Hamming distance of $\vec x$ to the class associated with $\rho$ is $\geq 1$.
	\end{itemize}
\end{itemize}

Let us now present the second part of the proof:

	\begin{itemize}
		\item {\bf Decision lists.}
		      Every \cnf\ formula $\rho$ can be turned in linear time into an equivalent decision list $L$ 
		      (see Theorem 1 from \cite{Rivest87}). Accordingly, every reduction pointed out in the first part of the proof 
		      can be turned into a reduction such that the targeted representation is a decision list, and this concludes the proof.

		\item {\bf Random forests.}
		      We exploit the same idea as in the proof for the decision lists case. 
		      To get the result, it is enough to show that every \cnf\ formula $\rho$ can be
		      turned in linear time into an equivalent random forest $F$. The translation is as follows: given a \cnf\ formula $\rho = \bigwedge_{i=1}^k \delta_i$ over $X_n$, we associate
		      with it in linear time the random forest $$F = \{T_1, \ldots, T_k, \underbrace{0, \ldots, 0}_{k-1}\}$$ over $X_n$, where each $T_i$ ($i \in [k]$) is a decision tree over $X_n$ that
		      represents the clause $\delta_i$.

		      Each $T_i$ ($i \in [k]$) is a comb-shaped tree that can easily be generated in time linear in the size of $\delta_i$:
		      if $\delta_i$ is the empty clause, then return $T_i = 0$, else considering the literals $\ell$ of $\delta_i$ in sequence,
		      generate a decision node of the form $\mathit{ite}(x, 1, T_i^{\ell})$ (resp. $\mathit{ite}(x, T_i^{\ell}, 1)$ if $\ell$ is a negative literal $\overline{x}$ (resp. a positive literal $x$),
		      where $T_i^{\ell}$ is a decision tree for the clause $\delta_i \setminus \{\ell\}$.

		      Finally, by construction, the only subset of trees of $F$ that contains more that half of the trees and that can be consistent is $\{T_1, \ldots, T_k\}$
		      (every other subset of $F$ containing at least $k$ trees contains a tree reduced to $0$, and as such, is inconsistent). Accordingly, $F$ is
		      equivalent to the conjunction of all trees from $\{T_1, \ldots, T_k\}$. Since each $T_i$ ($i \in [k]$) is equivalent to the clause
		      $\delta_i$ of $\rho$, we get that $F$ is equivalent to $\rho$, as expected.

		\item {\bf  Boosted trees.}
		      Direct from 
		      the proof for the random forests case, given that a random forest $F = \{T_1, \ldots, T_m\}$ can be turned in linear time into
		      an equivalent boosted tree $B = \{\langle T_1, \frac{1}{m}\rangle, \ldots, \langle T_m, \frac{1}{m}\rangle\}$ where the weight of each tree is equal to $\frac{1}{m}$.

		\item {\bf Boolean multilayer perceptrons.}
		      It is not difficult to turn in polynomial time any \cnf\ formula $\rho = \bigwedge_{i=1}^k \delta_i$
		      over $X_n = \{x_1, \ldots, x_n\}$, such that $\rho$ does not contain any valid clause (this can be ensured efficiently), into a Boolean multilayer perceptron $P$ over $X_n$, that
		      is logically equivalent to $\rho$. One uses only three layers: as expected, the first one $V_1$ contains $n$ vertices (on per variable $x_i \in X_n$),
		      the last one $V_3$ contains a single vertex $v_{3,1}$, and the second layer $V_2$ contains $k$ vertices, one per clause in $\rho$. 
		      The output of $v_{3,1}$ is the output of $P$.
		      Let $\delta_i$ be any clause of $\rho$ and let $v_{2,i}$ be the corresponding vertex. 
		      For every vertex $v_{1,j}$ ($j \in [n]$) from the first layer $V_1$ that is associated with a variable $x_j \in X_n$, 
		      the edge $(v_{1,j}, v_{2,i}) \in E$ that connects $v_{1,j}$ to $v_{2,i}$ is labelled by $w(v_{1,j}, v_{2,i}) = 0$ if $x_j$ does not occur in $\delta_i$, 
		      by $w(v_{1,j}, v_{2,i}) = 1$ if $x_j$ is a positive literal of $\delta_i$, and by $w(v_{1,j}, v_{2,i}) = -1$ if $\neg x_j$ is a negative
		      literal of $\delta_i$.  The value of the bias $b_{2,i}$ is the number of negative literals in $\delta$, minus $1$. 
		      By construction, we have $o_{v_{2,i}}(\vec x) = 1$ if and only if $\vec x$ satisfies the clause $\delta_i$ associated with $v_{2,i}$.
		      Now, for every vertex $v_{2,i}$ ($i \in [k]$) of the second layer $V_2$, the edge $(v_{2,i}, v_{3,1}) \in E$ 
		      that connects $v_{2,i}$ to $v_{3,1}$ is labelled by $w(v_{2,i}, v_{3,1}) = 1$, and the bias $b_{3,1}$ is set to $-k$
		      Accordingly, the output of $o_{v_{3,1}}$ of $P$ is $1$ if and only if every 
		      clause $\delta_i$ of $\rho$ is satisfied by $\vec x$, or stated otherwise, if and only if $\vec x$ is a model of $\rho$.

		\item {\bf Binarized neural nets.}
		      With a \cnf\ formula $\rho = \bigwedge_{i=1}^k \delta_i$
		      over $X_n = \{x_1, \ldots, x_n\}$, such that $\rho$ does not contain any valid clause, we associate in polynomial time a BNN $N$
		      with $2n$ inputs in $\{-1, 1\}$ and an output value in $\{1, 2\}$ such that for any $\vec x \in \mathbb{B}^n$, $\vec x$ is a model of $\rho$
		      if and only if its translation $\mathit{transl}(\vec x) \in \{-1, 1\}^{2n}$ is classified as a positive instance by $N$ (i.e., the output value of $N$ is $2$).
		      
		      The translation mapping $\mathit{transl}: \mathbb{B}^n \rightarrow \{-1, 1\}^{2n}$ can be computed in linear time and is defined as follows:\\

		      \begin{tabular}[t]{cc}
		           $\mathit{transl}(x_1, \ldots, x_n)$ & \\
			  $= (\underbrace{2x_1-1, 2x_1-1}_2, \ldots, \underbrace{2x_n-1, 2x_n-1}_2)$ &\\
		      \end{tabular}

		      i.e., for $i \in \{0, \ldots, n-1\}$,
		      the $(2i+1)^{th}$ coordinate and the $2(i+1)^{th}$ coordinate of $\mathit{transl}(x_1, \ldots, x_n)$ are equal to $2x_{i+1}-1$.
		      $N$ consists of a single intermediate block, i.e., $m=1$, so that the total number of blocks is $d = 2$. 
		      The number of inputs of the first block is $2n$ and the number of outputs
		      of this block is $k$, the number of clauses of $\rho$.

		      The key idea of our translation from \cnf\ to BNN is to consider each of the $k$ clauses $\delta$ of $\rho$ individually
		      and compute the difference between the number of literals of $\delta$ falsified by $\vec x$ and the number of literals of $\delta$ satisfied by $\vec x$,
		      which is noted as follows:
		      $$
			      \mathit{diff}(\delta, \vec x) = |\{\ell \in \delta : \vec x \models \overline{\ell}\}| - |\{\ell \in \delta : \vec x \models \ell\}|.
		      $$

		      Obviously enough, $\vec x$ does not satisfy $\delta$ precisely when $\mathit{diff}(\delta, \vec x)$ is the number of literals occurring in $\delta$.

		      The first operation realized by the BNN is a linear transformation. Thus one must define $\vec{A}_1$ and $\vec{b}_1$.
		      Basically, one wants to use this transformation to store in the output $\vec{y} = \vec{A}_1 \mathit{transl}(\vec x) + \vec{b}_1$ 
		      some information about the satisfaction of the clauses of $\rho$ by $\vec x$,
		      so that $y_i$ ($i \in [k]$) corresponds to the clause $\delta_i$ of $\rho$.
		      Because clauses are in general not built upon all variables, we need to add a mechanism to avoid considering
		      the variables that do not appear in a clause. This is achieved within $\mathit{transl}$ by duplicating the coordinates
		      of the input vector $\vec x$ (in addition to translating them from $\mathbb{B}$ to $\{-1, 1\}$).
		      Indeed, if a literal $\ell$ over $x \in X_n$ is not present in a clause $\delta_i$, its contribution to $y_i$ is expected to be $0$,
		      which is achieved by multiplying the two coordinates associated with $x$ in $\mathit{transl}(\vec x)$ by $-1$ and $1$, respectively.
		      If a literal $\ell$ over $x \in X_n$ occurs in $\delta_i$, we want its contribution to $y_i$ to be set to $1$ if the
		      literal falsifies the clause and to $-1$ if it satisfies the clause.
		      By summing up the contribution of each literal in that way, we obtain the expected result.
		      Formally, the matrix $\vec{A}_1[i]$ for $i \in [k]$ is defined as:
		      $$
			      \vec{A}_1[i] = (\tau_{x_1}^1, \tau_{x_1}^2,\tau_{x_2}^1, \tau_{x_2}^2, \ldots,\tau_{x_n}^1, \tau_{x_n}^2) \text{, where}
		      $$
		      $$
			      \text{for each } j \in [n] \left\{
			      \begin{array}{ll}
				      \tau_{x_j}^1 = -\tau_{x_j}^2     & \text{if } x_j \notin \mathit{Var}(\delta_i)         \\
				      \tau_{x_j}^1 = \tau_{x_j}^2 = 1  & \text{if } \overline{x_j} \in \delta_i\\ 
				      \tau_{x_j}^1 = \tau_{x_j}^2 = -1 & \text{if } x_j \in \delta_i\\ 
			      \end{array}
			      \right.
		      $$

		      Then one sets $\vec{b}_1[i]$ ($i \in [k]$) to $-2 \times |\delta_i| + 1$.
		      Overall, we get that for every $i \in [k]$, $y_i = 1$ if the clause $\delta_i$ is falsified by $\vec x$  and $y_i < 0$ if $\delta_i$ is satisfied by $\vec x$.
		      As to batch normalization, for every $i \in [k]$, we set the parameters $\alpha_{1_i} = \nu_{1_i}$ to $1$, and $\mu_{1_i} = \gamma_{1_i}$ to $0$,
		      so that the output of the transformation coincides with its input. Finally, the binarization transformation takes place. Clearly enough, the output
		      of the internal block of $N$ is a vector ${\vec x'} = (x'_1, \ldots x'_k) \in \{-1, 1\}^k$ such that
		      for every $i \in [k]$, $x'_i = 1$ if $\vec x$ falsifies $\delta_i$ and $x'_i = -1$ if $\vec x$ satisfies $\delta_i$.

		      By construction, this vector ${\vec x'}$ is the input of the output block $O$ of $N$.
		      Let us recall that the output value $o$ of $O$ is the output of $N$, it is a value in $\{1, 2\}$ indicating whether or not the input $\mathit{transl}(\vec x)$ is
		      classified positively by $N$. 
		      The linear transformation used in $O$ is given by $\vec{A}_d \in \{-1, 1\}^{2 \times k}$ and $\vec{b}_d \in \mathbb{R}^2$ such that
		      $\vec{A}_d[1] = (\underbrace{1, \ldots, 1}_k)$ and $\vec{A}_d[2] = \underbrace{(-1, \ldots, -1)}_k$, while $\vec{b}_d[1] = 2k$ and $\vec{b}_d[2] = 1$.
		      As such, when the instance $\vec x$ satisfies $\rho$, it satisfies every clause of it and the coordinate $w_2$ of the output $\vec{w}$ of the linear transformation
		      ${\bf w} = \vec{A}_d {\vec x'} + \vec{b}_d$ is equal to $k+1$, while its coordinate $w_1$ is equal to $k$. Contrariwise,  when the instance $\vec x$
		      does not satisfy $\rho$, we have $w_2 \leq k$ and $w_1 \geq k+1$. As a consequence, the last transformation of $O$ (the ARGMAX layer)
		      will return $o = 2$ when $\vec x$ satisfies $\rho$ and $o = 1$ otherwise.
		      Thus, $N$ can be viewed as a representation of the \cnf\ formula $\rho$ modulo the translation mapping $\mathit{transl}$, in the sense $\forall \vec x \in \mathbb{B}^n$, 
		      $\rho(\vec x) = c$ if only if $N(\mathit{transl}(\vec x)) = c+1$. 

	\end{itemize}
	
Finally, with each reduction given in the first part of the proof, we can associate another polynomial reduction where $\rho$ is a \dnf\ classifier.
This comes from two points: (1) the duality relating \cnf\ classifiers to \dnf\ classifiers stating that $\vec x \in \mathbb{B}^n$ is a positive instance of a concept represented by a \cnf\ classifier 
$\rho$ over $X_n$ if  and only if $\vec x$ is a negative instance of the (complementary) concept represented by a \dnf\ classifier $D$ that is equivalent to $\neg \rho$ (obviously, such a $D$ is computable from $\rho$ in linear time by applying De Morgan's laws); 
and (2) the fact that all the XAI queries we have considered must be able to take account for both positive and negative instances. This concludes the proof.
\end{proof}

The last two propositions thus show the existence of {\it a large computational intelligibility gap} between the families of classifiers at hand. Since
each of the $9$ XAI queries is tractable for the family of decision trees, {\it decision trees can be considered as highly intelligible in comparison to
the other families of classifiers considered in the paper}. At the other extremity of the spectrum, \dnf\ classifiers, decision lists, random forests, boosted
trees, Boolean multilayer perceptrons, and binarized neural nets appear as poorly intelligible since none of the $9$ XAI queries is tractable for any of those families.


Notably, the results reported in the last two propositions differ significantly from those presented in a number of previous papers 
where an equivalence-preserving polynomial-time translation (alias an encoding) from a given family \L\ of Boolean classifiers  
to \cnf\ formulae is looked for (see e.g., \cite{DBLP:conf/ijcai/Narodytska18,DBLP:conf/aaai/NarodytskaKRSW18,DBLP:conf/iclr/NarodytskaZGW20}).
Determining such translations permits  to take advantage of automated reasoning techniques
for addressing XAI queries, given that existing solvers for Boolean representations are most of the time based on the \cnf\ format.
Here, we have looked for polynomial-time translations from the language of \cnf\ formulae to the languages \L\ of the classifiers we focus on, 
in order to prove that the XAI queries are computationally hard whenever the input classifier is in \L. This is quite a different, yet complementary perspective.
Indeed, when such translations exist, leveraging \cnf\ encodings to address XAI queries 
makes sense from a computational standpoint, i.e., it is not using a sledgehammer to crack a nut. However, 
those translations are not guaranteed to exist for every family \L\ of classifiers, so that it can be the case that an XAI query is tractable for the 
\L\ while being computationally hard for \cnf\ classifiers. 
Accordingly, our study shows that this is precisely what happens with decision trees. For this family of classifiers,
it is meaningful to develop algorithms for addressing XAI queries that are directly based on the representations at hand (decision trees), instead of 
designing \cnf\ encodings.

\section{Conclusion}\label{sec:concl}

In this paper, we have investigated from a computational perspective the intelligibility of 
several families of Boolean classifiers: decision trees, \dnf\ formulae,
decision lists, random forests, boosted trees, Boolean multilayer perceptrons, and binarized neural nets.
The computational intelligibility of a family of classifiers has been evaluated as the
set of XAI queries that are tractable when the classifier at hand belongs to the family. 
Considering a set of $9$ XAI queries as a base line,
we have shown the existence of a large computational intelligibility gap between the families of classifiers. Roughly speaking, the results obtained show that
though decision trees are highly intelligible, the other families of classifiers we have focused on are not intelligible at all.
This coheres with the commonly shared intuition that ``decision trees are interpretable and other machine learning classifiers are not’’, but more than that, 
our results give some formal ground to this intuition.


This work completes the study \cite{Audemardetal20} that focuses on designing tractable cases for a superset of the $9$ XAI queries considered here, using
knowledge compilation techniques. The fact that each of the $9$ XAI queries is intractable ({\sf NP}-hard) when the Boolean classifier $\rho$ under consideration
is unconstrained 
justifies the need to look for specific representations of  circuits
into languages ensuring the tractability of those queries and for translation (``knowledge compilation'') techniques for turning
classifiers into representations from such tractable languages, as it has been done in \cite{Audemardetal20}. 

\smallskip

Various perspectives of research emerge from this paper. Notably, in the present study, we have focused on representation languages which are complete for propositional logic. 
In other words, such representation languages are expressive enough to cover the concept class $\mathcal F_n$ of all Boolean functions over $n$ variables.
From the viewpoint of computational learning theory, this means that the VC dimension \cite{Vapnik1974} of all families of Boolean classifiers considered in this paper is equal to $2^n$. 
This in turn implies that these families are \emph{not efficiently} PAC learnable \cite{Valiant1984} because their sample complexity is exponential in $n$.
In fact, it is well-known that the minimal size of any Boolean multilayer neural network representing all $n$-dimensional Boolean functions must be exponential in $n$ \cite[Theorem 20.2]{Shalev2014}.
So, in order to analyze the interpretability of common Boolean classifiers that can be trained using a reasonable amount of data samples, we need to 
focus on representation languages for which the VC dimension of the corresponding concept class is polynomial in $n$. 
As a prototypical example, for the class of decision lists defined over monomials of size at most $r$, the VC dimension is polynomial in $n$ when $r$ is constant; in fact, $r$-decision lists are efficiently PAC learnable \cite{Rivest87}.
Of course, not all representation classes with polynomial VC dimension admit a polynomial-time learning algorithm, but virtually all Boolean classifiers used in practice 
are defined from concept classes with polynomial sample complexity. Thus, it is clear that the computational intelligibility of incomplete classes 
is worth being studied. 

Enlarging the set of XAI queries that are used for the intelligibility assessment is another dimension for further research. 
Among the computational queries that could be added to the intelligibility map is the ability (or not) to compute SHAP scores in a tractable way,
as investigated recently in \cite{DBLP:conf/aaai/ArenasBBM21,DBLP:conf/aaai/BroeckLSS21}.
Deriving more fine-grained complexity results (i.e., not restricted to {\sf NP}-hardness in the broad sense) and determining whether 
the answers to some XAI queries can be approximated efficiently (under some guarantees on the quality of the approximation achieved) would be useful as well.
Finally, experiments will be also needed to determine to which extent the XAI queries we have considered in the paper are hard to be addressed in practice.

\newpage
\section*{Acknowledgements}
Many thanks to the anonymous reviewers for their comments and insights. 
This work has benefited from the support of the AI Chair EXPE\textcolor{orange}{KC}TA\-TION (ANR-19-CHIA-0005-01) of the French National Research Agency (ANR).
It was also partially supported by TAILOR, a project funded by EU Horizon 2020 research and innovation programme under GA No 952215.

\bibliographystyle{kr}
\bibliography{biblio}

\end{document}